\documentclass[11pt,a4paper]{article}
\usepackage[T1]{fontenc}
\usepackage[utf8]{inputenc}
\usepackage{lmodern}
\usepackage[margin=25mm]{geometry}
\usepackage{amsmath,amssymb,amsthm,bm,mathtools}
\usepackage{graphicx,booktabs,tabularx,array,microtype,enumitem}
\usepackage[numbers,sort&compress]{natbib}
\usepackage{tikz}
\usetikzlibrary{arrows.meta,positioning}
\usepackage[hidelinks]{hyperref}
\usepackage[font=small,labelfont=bf]{caption}
\setlist{nosep,leftmargin=*}
\newcommand{\R}{\mathbb{R}}
\newcommand{\sat}{\operatorname{sat}}
\newcommand{\diag}{\operatorname{diag}}
\newcommand{\sgn}{\operatorname{sgn}}
\newcommand{\norm}[1]{\left\lVert#1\right\rVert}
\newtheorem{proposition}{Proposition}
\title{Paired-Outcome Learning for Selective Spacecraft Attitude Capture under Torque Limits}
\author{%
\normalsize Donghe Chen$^{1}$, Jia Deng$^{2}$, Zehao Jia$^{2}$,\\[0.3em]
\normalsize Lin Cheng$^{1,*}$, Shengping Gong$^{1}$\\[0.6em]
\small $^{1}$School of Astronautics, Beihang University, Beijing 100191, China\\
\small $^{2}$Astroup(Beijing) Technology Co., Ltd.\\
\vspace{0.5em}
\footnotesize \href{mailto:paoshou@buaa.edu.cn}{\texttt{paoshou@buaa.edu.cn}},
\href{mailto:dengjia@astroup.com}{\texttt{dengjia@astroup.com}},
\href{mailto:jiazehao@astroup.com}{\texttt{jiazehao@astroup.com}}\\
\footnotesize $^{*}$Corresponding author: \href{mailto:chenglin5580@buaa.edu.cn}{\texttt{chenglin5580@buaa.edu.cn}}\\
\footnotesize \href{mailto:gongsp@buaa.edu.cn}{\texttt{gongsp@buaa.edu.cn}}}
\hypersetup{pdfauthor={Donghe Chen, Jia Deng, Zehao Jia, Lin Cheng, Shengping Gong}}
\date{}
\begin{document}
\maketitle
\vspace{-1.1em}
\begin{abstract}
Torque-limited spacecraft attitude capture requires adapting the maneuver speed to uncertain braking authority. Faster settings can reduce average tracking error while worsening individual maneuvers, making selective modification of an existing controller a baseline-relative learning problem. A paired-outcome method combines three mechanisms: physically interpretable continuation policies with different speed envelopes, separate prediction of task failure and successful-pair tracking gains, and context-dependent rejection that executes the unchanged baseline. Complete paired simulations provide labels under common plant parameters and measurement seeds; training, validation, and calibration are separated by context. Basilisk experiments use 108 development contexts, 324 labeling rollouts, and 624 evaluation episodes across 48 test contexts, including disclosed post-freeze diagnostics. The selective controller reduces aggregate integrated attitude error by 4.67\% relative to the branch baseline, with 17 improvements, 31 exact ties, and no observed regressions. Always using the fast alternative achieves a 4.81\% reduction but worsens 11 cases. Gains occur in principal-axis, near-axis, braking-boundary, and selected noise/drift maneuvers; general coupled cases retain the baseline. Ablations show no additional aggregate benefit from branch commitment, active querying, or score subtraction. The evidence supports a restricted empirical trade-off between average improvement and individual regression, without establishing a closed-loop safety guarantee.
\end{abstract}

\section{Introduction}\label{sec:introduction}
Autonomous spacecraft maneuvering requires decisions that remain effective when actuation authority and model knowledge are limited. Mission design uses physical structure to identify useful trajectories, as illustrated by energy-domain construction of gravity-assist escape trajectories~\cite{fu2026}; attitude guidance must additionally coordinate pointing requirements and rotational motion~\cite{calaon2025,li2025}. The present study concerns the latter problem at the rigid-body capture level: bringing a large attitude error and nonzero angular rate into a terminal neighborhood under bounded torque and uncertain inertia. When a disturbance nearly consumes the braking torque, a speed setting that appears favorable initially can make a later target crossing harder to arrest. The resulting difficulty is not simply insufficient feedback gain. It is deciding which complete continuation to execute from the observed context, while retaining a functioning controller where the benefit of modification is uncertain.

Existing approaches can be organized into model-based feedback adaptation, constrained command or trajectory modification, and data-driven policy improvement. Model-based feedback designs exploit the dynamics and, where necessary, discrete attitude memory: the topology of rotation and the unwinding phenomenon constrain continuous global stabilization~\cite{bhat2000}, while quaternion hybrid feedback provides a principled tracking construction under its assumptions~\cite{mayhew2011}. Adaptation is mechanism-specific; for example, adaptive notch filtering and gain tuning address rigid--elastic coupling in launch vehicles~\cite{jiang2026}, a different phenomenon from the rigid-body braking problem studied here. Reference and command governors modify inputs to an existing loop~\cite{bemporad1998,garone2017}, including explicit nonlinear constructions~\cite{garone2014} and spacecraft governors for nonconvex pointing constraints~\cite{nicotra2019}. Model predictive control, barrier-function quadratic programs, and predictive safety filters provide complementary constraint-handling frameworks~\cite{mayne2000,ames2017,wabersich2021}. Their assumptions and online constructions differ: an empirical speed envelope does not inherit their invariance or feasibility guarantees. The relevant opportunity is therefore a restricted modification of an existing capture loop, rather than a claimed replacement for these frameworks.

Data-driven improvement can select modifications using observed context, but its performance must be assessed relative to a specific incumbent. Learning-based predictive control and safe-learning research combine learned information with models and uncertainty descriptions~\cite{hewing2020,brunke2022}, with stability guarantees available in explicitly specified settings~\cite{berkenkamp2017}. High-confidence policy improvement, robust baseline-regret minimization, and baseline bootstrapping address different forms of uncertainty in departures from an incumbent~\cite{thomas2015,petrik2016,laroche2019}. These results motivate baseline-relative reasoning but do not directly certify a spacecraft selector trained from paired simulations. A candidate's label is the outcome of a complete feedback continuation, not an instantaneous torque advantage. Task failure and tracking gain also carry different meanings: a low-error successful maneuver and a failed maneuver should not be combined through an arbitrary large cost penalty. Consequently, outcome construction is part of the control design, and a favorable average cannot establish that individual maneuvers improve.

Context-dependent algorithm selection and rejection provide a useful way to organize this decision. Algorithm-selection models and solver portfolios exploit complementary instance-dependent performance~\cite{rice1975,xu2008}; selective prediction adds an explicit abstention option and studies the resulting risk--coverage trade-off~\cite{cortes2016,geifman2017,geifman2019}. In control, abstention must execute a specified feedback continuation. Three issues remain for the present portfolio: physically meaningful alternatives must expose the speed--braking trade-off; matched complete-episode labels must distinguish failure from gain; and deployment must compare the predicted benefit with the option of retaining the baseline. These are the concrete design questions addressed below. They do not imply that existing selective-learning theories are deficient in their own settings or that a learned score is a control-safety certificate.

To address these issues, paired-outcome learning is used to select among a frozen branch controller and two interpretable speed alternatives. The baseline first provides the common state history and initial capture branch. The alternatives adjust the reference-speed envelope, while separate predictors estimate failure scores and gains from successful pairs. A rejection rule then determines whether the entire alternative continuation is executed; rejecting a change invokes the unchanged baseline, including its inherited branch decisions. The technical contributions are summarized as follows:
\begin{enumerate}
\item A branch-conditioned continuation portfolio makes the physical speed choices and the no-change action executable under the same torque limits.
\item A context-level paired labeling and learning formulation separates task failure from baseline-relative tracking improvement.
\item Selective execution provides verifiable baseline recovery on rejection, with its performance trade-off assessed through fixed-policy counterfactuals and module ablations.
\end{enumerate}
\noindent These are task-specific engineering contributions built from established control and learning tools. Basilisk evaluation shows a 4.67\% aggregate error reduction with 17 gains and 31 exact ties relative to the branch baseline, whereas fixed fast control produces 11 regressions despite a slightly lower mean error. The experiments also identify components without an additional demonstrated benefit, rather than treating every added module as an independent innovation.

Section~\ref{sec:problem} formulates the dynamics, baseline-relative objective, and three design issues. Section~\ref{sec:method} develops the continuation policies, paired predictors, and selective execution rule. Section~\ref{sec:experiments} evaluates control performance, policy-selection behavior, mechanism ablations, and applicability limits. Section~\ref{sec:conclusion} summarizes the findings and remaining validation needs. Appendices give implementation, data-acquisition, and reproduction details.

\section{Problem formulation and design issues}\label{sec:problem}
The objective is selective improvement of a fixed-target capture controller. This section defines the rigid-body plant and available measurements, distinguishes task completion from relative performance, and derives the three design issues that organize the method.

\subsection{Spacecraft dynamics and available information}
Let $\bm\sigma\in\R^3$ be the modified Rodrigues parameter (MRP) attitude error and $\bm\omega\in\R^3$ the body angular velocity. With the simulator's conventions,
\begin{align}
\dot{\bm\sigma}&=\tfrac14\big[(1-\bm\sigma^T\bm\sigma)I+2[\bm\sigma]_\times+2\bm\sigma\bm\sigma^T\big]\bm\omega,\label{eq:mrp}\\
J\dot{\bm\omega}&=\bm u+\bm d(t)-\bm\omega\times J\bm\omega.\label{eq:euler}
\end{align}
The unknown inertia $J$ is diagonal and constant. The disturbance $\bm d$ and applied torque $\bm u$ are expressed in the body frame. The nominal inertia is $J_n=\diag(900,800,600)\;\mathrm{kg\,m^2}$. Online inputs comprise measured MRP and angular velocity, $J_n$, torque limits, the control period, and last applied torque. Noise is added in MRP and rate coordinates. The reported suite has no measurement delay, actuator lag, or slew constraint.

Componentwise saturation enforces $|u_i|\leq\bar u_i$. Five test families use $\bar u_i=0.3\;\mathrm{N\,m}$; the ordinary family samples a common limit in $[0.25,0.5]\;\mathrm{N\,m}$. The evaluated rate limit is $2^\circ/\mathrm{s}$. Torque clipping is an implementation property; sampled rate satisfaction is not a continuous-time certificate.

The unsigned attitude error is $a(t)=2\arccos|q_0(t)|$, with
\begin{equation}
q=\begin{bmatrix}(1-\norm{\bm\sigma}^2)/(1+\norm{\bm\sigma}^2)\\2\bm\sigma/(1+\norm{\bm\sigma}^2)\end{bmatrix}.
\end{equation}
Shadow MRPs are used when needed. Controller equations use radians and seconds; reported attitude integrals use degrees.

\subsection{Capture objective and baseline-relative evaluation}
Over a horizon $T$, task completion requires $a(t)<1^\circ$ and $\norm{\bm\omega(t)}<0.1^\circ/\mathrm{s}$ at every sample in $[T-60,T]$, and a full-episode rate maximum no greater than $2^\circ/\mathrm{s}$ up to a numerical tolerance of $10^{-4}\,{}^\circ/\mathrm{s}$. Let $F_a\in\{0,1\}$ indicate failure under portfolio action $a$. The tracking and effort metrics are
\begin{equation}
C_a=\int_0^T a_a(t)\,dt\quad[\mathrm{deg\,s}],\qquad E_a=\int_0^T\norm{\bm u_a(t)}^2dt.\label{eq:metrics}
\end{equation}
Trapezoidal integration uses the original $0.1\,$s logs. Settling time is the first sample after the final terminal-threshold violation, rather than the first entry into the terminal region.

At the first capture-planning event, observable context $x$ selects one policy from $\Pi=\{\pi_0,\pi_1,\pi_2\}$; $\pi_0$ is the existing controller. Let $z$ denote the actual plant, initial state, and measurement-noise realization, which determine the closed-loop costs $C_a(z)$ and task labels $F_a(z)$. A selector $g:x\mapsto a$ is evaluated on a finite suite through
\begin{align}
\overline\Delta_g&=\frac1N\sum_{i=1}^{N}\left[C_0(z_i)-C_{g(x_i)}(z_i)\right],\label{eq:relativeobjective}\\
R_g&=\sum_{i=1}^{N}\mathbf1\{C_{g(x_i)}(z_i)>C_0(z_i)+\epsilon_C\},\\
N_g^{\rm new}&=\sum_{i=1}^{N}\mathbf1\{F_{g(x_i)}(z_i)=1,\ F_0(z_i)=0\},\label{eq:regressioncounts}
\end{align}
where $\epsilon_C=10^{-6}\;\mathrm{deg\,s}$ distinguishes numerical ties. The design seeks positive $\overline\Delta_g$ while limiting $R_g$ and new task failures; these are empirical evaluation objectives, not constraints proved for unseen maneuvers. In particular, $\overline\Delta_g>0$ does not imply $R_g=0$. The no-change action supplies a defined comparison and fallback, rather than a claim that its performance is globally optimal.

\subsection{Three design issues for selective capture}\label{sec:issues}
\paragraph{Speed adaptation depends on braking authority.}
For a scalar projected motion with constant available deceleration $b>0$, stopping a rate $v$ requires distance $r_{\rm stop}=v^2/(2b)$ and has sensitivity
\begin{equation}
\frac{\partial r_{\rm stop}}{\partial b}=-\frac{v^2}{2b^2}.\label{eq:stopsensitivity}
\end{equation}
This elementary scalar calculation explains why uncertain weak braking can make the target crossing sensitive to speed choice. In three-dimensional rigid-body motion, gyroscopic coupling, changing disturbance estimates, and branch selection prevent interpreting this proxy as a global capture certificate. The first design issue is to construct alternatives with interpretable speed--braking differences while retaining the existing feedback architecture.

\paragraph{Immediate control changes do not label complete-task benefit.}
Changing a speed envelope changes subsequent states, braking phases, and possibly branch updates, so an instantaneous acceleration estimate does not determine $C_0-C_a$. Comparing unrelated cases also confounds policy differences with plant and initial-condition differences. The second issue is to evaluate all alternatives from a common context under the same $z_i$, while keeping failure labels separate from successful-pair gain targets. Case-level pairing addresses the label definition; it does not remove the need for adequate coverage of difficult maneuvers.

\paragraph{Estimated benefit must be weighed against retaining the incumbent.}
A predictor can rank two alternatives even when neither improves on $\pi_0$. Moreover, a gain model trained only on successful pairs does not describe task failure. The third issue is to combine relative-gain evidence, failure scores, and a context-support check with an executable no-change decision. The next section resolves these issues at the architecture and implementation level. Its accepted actions remain empirically assessed; only rejection admits the exact baseline-identity statement given below.

\section{Paired-outcome selective capture}\label{sec:method}
The controller follows the three issues identified in Section~\ref{sec:issues}: construct physically interpretable continuations, learn matched complete-task outcomes, and reject changes when their evidence is insufficient. Figure~\ref{fig:architecture} shows the online information flow; estimation and branch planning remain part of the inherited capture loop.

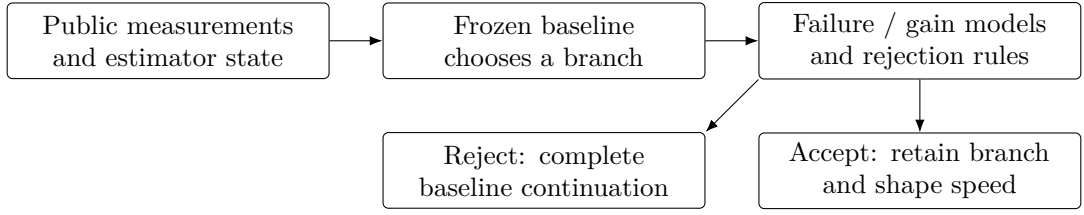
\begin{figure}[t]\centering
\begin{tikzpicture}[>=Latex,node distance=7mm and 7mm,
box/.style={draw,rounded corners=2pt,align=center,font=\small,minimum height=10mm,text width=40mm}]
\node[box] (obs) {Public measurements\\and estimator state};
\node[box,right=of obs] (branch) {Frozen baseline\\chooses a branch};
\node[box,right=of branch] (sel) {Failure / gain models\\and rejection rules};
\node[box,below=of sel] (change) {Accept: retain branch\\and shape speed};
\node[box,below=of branch] (base) {Reject: complete\\baseline continuation};
\draw[->] (obs)--(branch);\draw[->] (branch)--(sel);
\draw[->] (sel)--(change);\draw[->] (sel.south west)--(base.north east);
\end{tikzpicture}
\caption{One portfolio decision at the first capture-planning event. Rejection preserves the baseline's existing learned branch decisions. Acceptance changes a continuation policy, not one isolated torque sample.}\label{fig:architecture}
\end{figure}

\subsection{Branch-conditioned continuation policies}\label{sec:baseline}
The first mechanism exposes a small physical choice: retain the baseline speed law, use a balanced envelope, or use a faster envelope. Shared estimation and the initial branch decision are inherited so that the outer learner selects complete, directly comparable continuations.

\paragraph{Shared estimator and capture trigger.}
A nominal momentum observer estimates a lumped disturbance:
\begin{align}
\widehat{\bm d}_k&=\ell(J_n\bm\omega_k-\widehat{\bm h}_k),\\
\widehat{\bm h}_{k+1}&=\widehat{\bm h}_k+\Delta t\big(\bm u_{k-1}-\bm\omega_k\times J_n\bm\omega_k+\widehat{\bm d}_k\big),\label{eq:observer}
\end{align}
where $\ell=0.1\;\mathrm{s^{-1}}$ and $\widehat{\bm h}_0=J_n\bm\omega_0$. Outside capture intervention, all actions use
\begin{equation}
\bm u_{\rm normal}=\sat_{\bar{\bm u}}(-5\bm\sigma-50\bm\omega-\widehat{\bm d}).\label{eq:normal}
\end{equation}
These custom gains, with units producing torque, differ from the simulator example's separate feedback gains.

A fixed initial axis $\bm n$ is extracted from the vector part of $q$. The angle $\theta$ is unwrapped from $2\operatorname{atan2}(\bm q_v^T\bm n,q_0)$ and $v=\bm n^T\bm\omega$. An estimator retains per-axis correlated polygons $(J_i,d_i)$ initialized by $J_i\in[0.7,1.3]J_{n,i}$ and $|d_i|\leq(0.25,0.25,0.1)_i\;\mathrm{N\,m}$. Integrated Euler-equation strips update these polygons every $10\,$s (Appendix~\ref{app:estimator}).

Projected braking authority $b$ gives a stopping-distance proxy $v^2/(2b)$. Capture assessment begins when a current/one-step-ahead proxy exceeds $85\%$ of the positive remaining angle, or measured rate exceeds $1.9^\circ/\mathrm{s}$. The first $10\,$s of assessment uses braking. The branch planner then considers a nearby lifted target, one further $2\pi$ crossing in the current direction, and the committed target during replanning. The frozen baseline attempts a learned two-branch ranking with rejection. Otherwise it uses two finite parameter-scenario predictions, a $2\,$s step, a $900\,$s horizon, and a $30\,$s target-switch margin. Replanning is considered after $60\,$s if the disturbance estimate has changed by more than $0.03\;\mathrm{N\,m}$.

The inherited branch model used 108 earlier training contexts and 28 calibration contexts grouped into 26 source cases. Those costs are excluded from the 324 new portfolio-label rollouts. Finite two-scenario prediction and projected stopping bounds are engineering approximations, not a general $SO(3)$ capture certificate.

\paragraph{Alternative reference-speed envelopes.}
Let $\theta_\star$ be the inherited planner's target and $e=\theta-\theta_\star$. Reference-speed shaping uses
\begin{align}
s(e,b)&=\min\left\{\bar v,\frac{|e|}{\tau_T},\sqrt{2\gamma\max(b,0)|e|}\right\},\label{eq:envelope}\\
\bm\omega_r&=-\sgn(e)\bm n\,s(e,b),\\
\bm u&=\sat_{\bar{\bm u}}\left[J_n\frac{\bm\omega_r-\bm\omega}{\tau_v}+\bm\omega\times J_n\bm\omega-\widehat{\bm d}-K_\perp\bm\sigma_\perp\right].\label{eq:torque}
\end{align}
Here $\bm\sigma_\perp$ is the MRP residual relative to $q_r=[\cos(\theta/2),\bm n\sin(\theta/2)]$. It supplies transverse feedback; the actual plant remains three-dimensional. In a selected braking phase, $\bm\omega_r=0$ and $\tau_v=0.2\,$s until $|v|<0.01^\circ/\mathrm{s}$.

\begin{table}[t]\centering\small
\caption{Portfolio parameters: speed caps in degrees per second and time constants in seconds; $K_\perp=5$ in torque units per MRP component.}\label{tab:profiles}
\begin{tabular}{lrrrrl}\toprule
Action & $\bar v$ & $\gamma$ & $\tau_v$ & $\tau_T$ & Capture target update\\\midrule
0: baseline & 1.6 & 0.65 & 8 & 15 & inherited replanning\\
1: balanced & 1.6 & 0.80 & 4 & 10 & retain chosen target\\
2: fast & 1.8 & 0.85 & 3 & 8 & retain chosen target\\\bottomrule
\end{tabular}\end{table}

All policies share the same executed prefix and initial target choice. Action $0$ retains the full original baseline including replanning. Actions $1$ and $2$ retain the chosen target during scheduled replanning within the same capture episode. Normal feedback resumes when lifted error is below $3^\circ$ and rate below $0.15^\circ/\mathrm{s}$; a later assessment may choose a fresh target. The portfolio action itself is held for the entire simulation. The alternatives change both speed shaping and target commitment. Commitment is a candidate implementation choice; its incremental value is tested separately in Section~\ref{sec:supplement}.

\subsection{Paired failure and relative-gain learning}\label{sec:data}
The second mechanism learns which continuation is useful from the same decision context. Three full-episode simulations under common actual parameters and noise seeds produce $(F_0,F_1,F_2)$ and the corresponding tracking costs. Successful-pair regression uses
\begin{equation}
y_a=\log\frac{\max(C_0,1)}{\max(C_a,1)},\qquad a\in\{1,2\},\label{eq:gain}
\end{equation}
with numerical floor $1\;\mathrm{deg\,s}$, and fits only pairs with $F_0=F_a=0$. Failed continuations remain classification examples; they are not converted into arbitrary large regression costs. A context includes its common observation prefix. Neighboring frames are not split into nominally independent examples.

\paragraph{Context-level separation and coverage.}
A stratified generator creates 96 contexts: 48 training, 24 validation, and 24 calibration, assigned before policy outcomes are computed. Six families cover ordinary, principal-axis, near-axis, general coupled, braking-boundary, and noise/drift maneuvers. Actual inertia can be used to design offline boundary cases but is never an online feature. Each context runs all three policies with the same actual plant and noise seed, yielding complete-episode outcomes.

The final set has 60 training, 24 validation, and 24 calibration contexts: 324 full rollouts plus 96 short prefixes. Three training contexts never reach a planning event. The main selector therefore fits 57 contexts, whose four symmetry copies produce 228 classifier rows, not 240 independent samples. The no-active-query version fits 45 planning contexts from its 48 training contexts. One fast-policy failure occurs in the development data; tail-risk coverage is consequently sparse.

Training-only acquisition adds 12 cases selected from 96 short baseline prefixes; its rule is documented in Appendix~\ref{app:query}. Active learning~\cite{settles2009} motivates this optional coverage/disagreement heuristic. Error distribution smoothing also emphasizes the distinction between sampling density and function complexity in low-dimensional regression~\cite{chen2026}; the present method does not implement that pruning algorithm. Additional labels are evaluated through a no-query ablation rather than presumed beneficial.

\paragraph{Observable features and symmetry.}
The 34-dimensional decision vector contains normalized lifted angle and projected rate, measured angular velocity, fixed axis, estimated disturbance, torque limits, forward/reverse braking estimates, transverse rate, the nominal gyroscopic term, per-axis polygon extrema, chosen-target distance, and braking status. Nine additional dimensionless combinations describe stopping reserve, axis dominance, transverse coupling, gyroscopic magnitude, disturbance magnitude/projection, target distance, braking status, and total rate (Appendix~\ref{app:features}).

The augmentation group is restricted to
\begin{equation}
\mathcal G=\{\diag(1,1,1),\diag(-1,-1,1),\diag(-1,1,-1),\diag(1,-1,-1)\}.\label{eq:group}
\end{equation}
These proper rotations preserve diagonal inertia and symmetric torque boxes. Vector features transform by $G$; signed disturbance-interval endpoints are negated and swapped as appropriate. Scalar target distances remain unchanged. No reflection, axis permutation, or full rotational invariance is assumed.

Random forests~\cite{breiman2001} and extremely randomized trees~\cite{geurts2006} motivate the tabular ensemble; implementation uses scikit-learn~\cite{pedregosa2011}. Each alternative has five context-bootstrapped ExtraTrees gain regressors: 64 trees per regressor, maximum depth 7, and minimum leaf size 2. Predictions are first averaged over $\mathcal G$, then across the five regressors, giving mean $\mu_a$ and standard deviation $s_a$. Three separate 96-tree failure classifiers use all available task labels, class-balanced weights, maximum depth 7, and minimum leaf size 2. Their symmetry-averaged outputs $r_a$ are scores, not calibrated failure probabilities. Set $\mu_0=s_0=0$.

With augmented feature map $\phi(x)\in\R^{43}$, the distance check is
\begin{equation}
D(x)=\frac1{\sqrt{43}}\min_{G\in\mathcal G,\ z\in\mathcal X_{\rm train}}
\norm{(\phi(Gx)-z)\oslash\bm s_x},\label{eq:distance}
\end{equation}
where $\bm s_x$ is the coordinatewise training standard deviation floored at 0.05. The training set includes symmetry copies. The validation-distance 95th percentile times 1.25 gives $D_{\max}=1.59358$.

\subsection{Selective execution and baseline recovery}\label{sec:selection}
The third mechanism compares a proposed change with the no-change action. Validation selects eligibility and gain thresholds, and a disjoint calibration split supplies an empirical score subtraction. The resulting policy is held for the episode; it is not a torque filter that repeatedly repairs the applied input.

Eligible alternatives satisfy
\begin{equation}
\mathcal A(x)=\{a\in\{1,2\}:r_a(x)\leq\rho,\ r_a(x)\leq r_0(x)+0.05\}.
\end{equation}
Define $h_a=\mu_a-\kappa s_a$ and $\widetilde a=\arg\max_{a\in\mathcal A(x)}h_a$ if the set is nonempty. The decision is
\begin{equation}
a^*(x)=\begin{cases}
\widetilde a,&D(x)\leq D_{\max},\ \mathcal A(x)\ne\varnothing,\ h_{\widetilde a}(x)-q>\delta,\\
0,&\text{otherwise}.
\end{cases}\label{eq:select}
\end{equation}
The validation grid uses $\rho\in\{0.02,0.1,0.25,0.5\}$, $\kappa\in\{0,1\}$, and $\delta\in\{0,0.02,0.05\}$. Rules introducing new validation failures are excluded; remaining rules minimize (total failures, summed IAE, number of nonbaseline selections) lexicographically. The selected parameters are $(\rho,\kappa,\delta)=(0.02,0,0.02)$. A positive dispersion penalty is therefore not supported by this validation choice.

With these choices fixed, the 24 calibration contexts use the uncorrected decision ($q=0$). For a proposed nonzero action, $\eta_i=h_{\widetilde a_i}(x_i)-y_{\widetilde a_i}$; an unsafe continuation uses surrogate $y=-2$. No-change contexts receive $\eta_i=0$. The held-out correction is
\begin{equation}
q=\max\{0,\eta_{(k)}\},\quad k=\min\{24,\lceil(24+1)0.9\rceil\}=23,\label{eq:cal}
\end{equation}
giving $q=0.003238652$. Calibration methodology motivates this order-statistic operation~\cite{angelopoulos2022}. Conformalized quantile regression~\cite{romano2019} constructs intervals using a different procedure, while methods beyond exchangeability~\cite{barber2022} address departures from standard calibration assumptions. Neither construction is implemented here. The stratified context construction, failure surrogate, and selective deployment are not accompanied by a proof of conditional improvement or closed-loop safety. It is an empirical score subtraction, not a 90\% control-safety guarantee.

\begin{figure}[t]\centering\fbox{\begin{minipage}{0.94\linewidth}\small
\textbf{Algorithm 1: Selective continuation}
\begin{enumerate}
\item Execute the common prefix; update observer and inertia--disturbance polygons.
\item Obtain a lifted target and braking status from the frozen branch planner; construct public context $x$.
\item Predict $r_a$, $\mu_a$, $s_a$, and $D(x)$ with symmetry-averaged models.
\item Apply~\eqref{eq:select}. For $a^*=0$, retain the complete baseline controller.
\item Otherwise use the chosen speed profile and retain the target within capture; resume normal feedback on exit. Hold the portfolio action until the episode ends.
\end{enumerate}
\end{minipage}}
\caption{The online decision. Offline stages create paired labels, fit predictors, choose rules on validation, and compute the correction on a disjoint calibration split.}\label{alg:online}
\end{figure}

\paragraph{Exact behavior on rejection.}
\begin{proposition}[Conditional baseline identity]\label{prop:identity}
Suppose action $0$ invokes exactly the baseline state updates and control law, with the same initial state, measurement sequence, and numerical schedule. If the selector chooses $0$, its full executed trajectory equals the baseline trajectory.
\end{proposition}
\begin{proof}
The common prefix produces identical plant and controller states at selection. Action $0$ changes no state update or applied-torque computation. Equality propagates sample by sample through the deterministic simulator with the same measurement sequence.
\end{proof}
This is a software/closed-loop identity, not a new stability theorem, and says nothing about accepted nonzero actions. A $900\,$s development trajectory and all 31 rejected test cases were checked for exact equality in every logged column. Prediction and decision invariance were checked over all 104 development contexts with available features. Shallow-tree inference occurs once per episode; inherited branch planning remains a separate computational cost.

\section{Numerical validation}\label{sec:experiments}
The experiments test the three design questions using native rigid-body simulations: control benefit under common physical conditions, the role of context-dependent rejection, and the effect of individual modules. Protocol and comparator definitions precede the results; per-case failures, adverse examples, and computational limitations are reported alongside average performance.

\subsection{Evaluation design and comparator definitions}
The native rigid-body plant uses Basilisk 2.12.0~\cite{kenneally2020}. Nominal mass ($750\,$kg), inertia, and torque-effector organization follow the official attitude-feedback example~\cite{basiliskexample}, with a custom control interface and disturbances. The environment uses Python 3.12, NumPy 2.3.5, scikit-learn 1.8.0, and joblib 1.5.3. Candidate controllers see no actual plant parameters.

\begin{table}[t]\centering\small
\caption{New portfolio data and evaluation budget. Symmetry copies are not independent contexts.}\label{tab:data}
\begin{tabular}{lrrl}\toprule
Stage & Contexts & Full rollouts & Purpose\\\midrule
Base training & 48 & 144 & preliminary fitting\\
Training-only queries & 12 & 36 & training augmentation\\
Validation & 24 & 72 & rule / fixed-policy selection\\
Calibration & 24 & 72 & score subtraction\\\midrule
Development total & 108 & 324 & excludes inherited branch data\\
Initial frozen test & 48 & 384 & eight configurations per case\\
Post-freeze diagnostics & 48 reused & 240 & five added configurations\\\bottomrule
\end{tabular}\end{table}

Each family supplies eight base-training and four validation/calibration cases; active querying adds two per family. The test uses eight new cases per family, disjoint generator offsets and seeds, and horizons of $600\,$s for ordinary cases and $900\,$s otherwise. Five families sample $J_i/J_{n,i}\in[0.8,1.2]$, angles typically $100$--$165^\circ$, and projected initial speed $1$--$1.7^\circ/\mathrm{s}$. Near-axis cases perturb a principal axis by Gaussian components scaled between 0.02 and 0.13. Coupled and noise/drift cases use a normalized Gaussian axis. Ordinary cases instead sample $20$--$140^\circ$ and component rates from $[-0.02,0.02]\;\mathrm{rad/s}$.

Boundary cases target scalar stopping-distance/remaining-angle ratios
\begin{equation}
\mathcal R_b=\{0.65,0.9,1,1.1,1.4,1.9,2.3,2.8\},
\end{equation}
using actual inertia offline; speed clipping to $[0.65,1.8]^\circ/\mathrm{s}$ can change the realized ratio. Their along-axis disturbance is $0.22$--$0.25\;\mathrm{N\,m}$ on every axis. Consequently, the two $z$-axis boundary test cases exceed the estimator's $0.1\;\mathrm{N\,m}$ prior on that axis; these are prior-mismatch cases, not cases satisfying a robust disturbance assumption. Noise/drift cases use MRP/rate Gaussian standard deviations $10^{-5}$ and $10^{-5}\;\mathrm{rad/s}$, and $0.004\;\mathrm{N\,m}$ disturbance sinusoids with periods $100$--$180\,$s. Ordinary cases also use these sensor-noise levels, constant disturbances sampled from $[-0.12,0.12]\;\mathrm{N\,m}$ (the $z$ component clipped to $[-0.08,0.08]$), and sinusoidal amplitudes $0.01\;\mathrm{N\,m}$ with periods $80$--$180\,$s. In the other non-boundary families, signed disturbance magnitudes are sampled at 0.88--0.99 times $(0.25,0.25,0.1)\;\mathrm{N\,m}$, with zero components when the generating axis component is zero. Exact configurations and seeds are retained in the evidence bundle.

The eight configurations are: the frozen branch baseline; the prior strict-domain speed controller; the validation-selected best fixed portfolio policy; a geometric lookup selecting fixed policies in three public-feature bins; the same learner without active queries; the main learner without $q$; greedy mean-gain selection without risk/distance/margin checks; and the full selective portfolio. The best fixed policy is action $0$, so it deliberately duplicates the branch baseline rather than adding an independent competitor. The geometric lookup selects action 2 in its principal-axis bin and action 0 elsewhere, with the same distance rejection. The prior strict-domain controller is an earlier engineering variant, not a published state-of-the-art method.

Models, parameters, source hashes, and test generation were frozen before the completed tests. The archived audit retains the corrected legacy-import record and confirms that the models and selection parameters did not change. Earlier development tests are not counted as additional independent validation.

Five post-freeze configurations add 240 complete rollouts on the same 48 test contexts: fixed balanced, fixed fast, commitment-only, selected-speed-only with inherited branch replanning, and native PD. No model or threshold is retuned. The total is 624 evaluation episodes across 13 configurations, including the validation-fixed duplicate of the branch baseline. These diagnostics reuse the test contexts; they are not another independent validation set. Native PD uses $K=3.5$, $P=30$, no disturbance observer, and no candidate correction, so it is a reference architecture rather than a matched ablation of intelligence. Every configuration uses the same measurements, physical torque limits, actual plant, and noise seed within a case.

The tests address three questions in order: whether selected continuations improve capture performance, whether the outer learner avoids harmful unconditional changes, and which components account for the observed result. Statistical intervals resample paired context indices 10,000 times with seed 834092; they describe case variation in this stratified suite, not repeated independent training or a mission population.

\subsection{Capture accuracy and applied-control cost}\label{sec:performance}
Selective execution improves aggregate tracking performance relative to the branch baseline while slightly increasing control effort (Table~\ref{tab:overall} and Fig.~\ref{fig:baselineoverview}). Mean IAE decreases from $26\,060.4$ to $24\,842.4\;\mathrm{deg\,s}$, a $4.67\%$ reduction; mean settling time decreases from $400.92$ to $385.16\,$s. Mean effort changes from $81.72$ to $81.96\;\mathrm{N^2m^2s}$, so the result is not energy saving. The selective controller completes all 48 tasks and its maximum sampled rate is $1.90109^\circ/\mathrm{s}$. All other portfolio variants also complete all 48 tasks; native PD completes none under the specified terminal criterion.

The gain is selective rather than uniform. There are 17 strict IAE improvements and 31 exact ties against the branch baseline, with no observed regression; the exploratory interval for mean absolute paired gain is $[731.7,1728.9]\;\mathrm{deg\,s}$ (Table~\ref{tab:paired}). Figure~\ref{fig:paired} retains every case. The comparison with prior gated speed has 11 wins, 32 ties, five losses, and an interval crossing zero. The $1.53\%$ aggregate advantage over geometric lookup includes eight wins and nine losses. These results establish neither universal dominance over those comparators nor a population guarantee.

\begin{table}[tbp]\centering\small\setlength{\tabcolsep}{4pt}
\caption{Unified outcomes on 48 matched contexts. IAE: deg s; settling: s; effort: $\mathrm{N^2m^2s}$. Regressions count IAE increases relative to the branch baseline. $\dagger$ denotes post-freeze diagnostics. Native PD has different gains and no disturbance observer; a missing settling mean indicates at least one unsettled episode.}\label{tab:overall}
\begin{tabular}{lrrrrr}\toprule
Configuration & Tasks & Mean IAE & Settling & Effort & Regressions\\\midrule
Native PD$^\dagger$ & 0/48 & 63,340.6 & --- & 98.44 & 47 \\
Branch baseline & 48/48 & 26,060.4 & 400.92 & 81.72 & 0 \\
Prior gated speed & 48/48 & 26,059.7 & 397.21 & 81.93 & 2 \\
Validation-fixed (= baseline) & 48/48 & 26,060.4 & 400.92 & 81.72 & 0 \\
Fixed balanced$^\dagger$ & 48/48 & 25,118.4 & 393.12 & 82.47 & 10 \\
Fixed fast$^\dagger$ & 48/48 & 24,806.0 & 393.55 & 82.85 & 11 \\
Geometric lookup & 48/48 & 25,228.4 & 389.94 & 81.91 & 1 \\
Greedy learned & 48/48 & 24,539.6 & 385.27 & 82.29 & 4 \\
Selective learned & 48/48 & 24,842.4 & 385.16 & 81.96 & 0 \\
Commitment only$^\dagger$ & 48/48 & 26,265.8 & 403.98 & 81.82 & 2 \\
Selected speed only$^\dagger$ & 48/48 & 24,842.4 & 385.16 & 81.96 & 0 \\
No active queries & 48/48 & 24,830.0 & 385.05 & 81.95 & 0 \\
No score subtraction & 48/48 & 24,812.1 & 385.28 & 81.98 & 0 \\
\bottomrule\end{tabular}\end{table}

\begin{figure}[tbp]\centering\includegraphics[width=\linewidth]{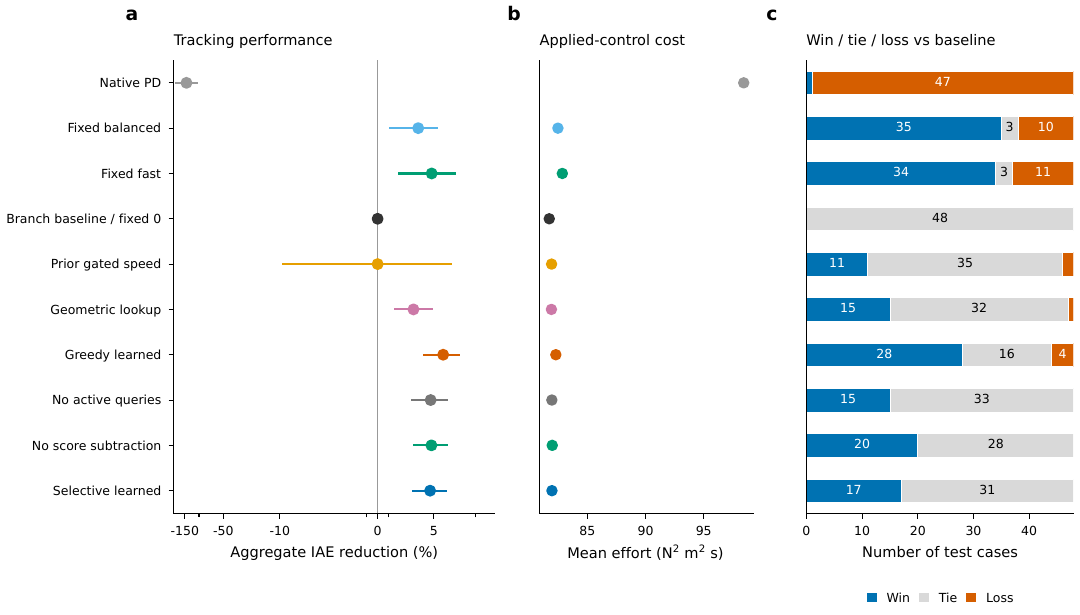}
\caption{Cross-baseline control outcomes on 48 matched contexts. \textbf{a}, Aggregate IAE reduction with exploratory 95\% paired-case bootstrap intervals (10,000 resamples, seed 834092); a symmetric-log axis compresses the native PD reference while retaining a linear region near zero. \textbf{b}, Mean integral of squared actual torque. \textbf{c}, Case-level IAE wins, ties, and losses relative to the branch baseline. Native PD has zero task completions, while the other displayed configurations have 48; an IAE win does not imply task completion. Fixed 0 and validation-fixed duplicate the branch baseline and are represented by its single row.}\label{fig:baselineoverview}\end{figure}
\begin{table}[t]\centering\small
\caption{Main method versus each configuration. Positive reductions favor the main method. Intervals are exploratory paired-case bootstrap intervals for absolute mean IAE improvement, in deg s.}\label{tab:paired}
\begin{tabular}{lrrr}\toprule
Comparison & Reduction (\%) & Win/tie/loss & 95\% interval\\\midrule
Branch baseline & 4.67 & 17/31/0 & $[731.7,\ 1,728.9]$ \\
Prior gated speed & 4.67 & 11/32/5 & $[-508.5,\ 3,274.6]$ \\
Geometric lookup & 1.53 & 8/31/9 & $[55.8,\ 779.2]$ \\
No active queries & -0.05 & 2/43/3 & $[-119.7,\ 112.3]$ \\
No score subtraction & -0.12 & 0/45/3 & $[-71.6,\ 0.0]$ \\
Greedy gain & -1.23 & 4/24/20 & $[-499.2,\ -139.0]$ \\

\bottomrule\end{tabular}\end{table}
\begin{figure}[t]\centering\includegraphics[width=\linewidth]{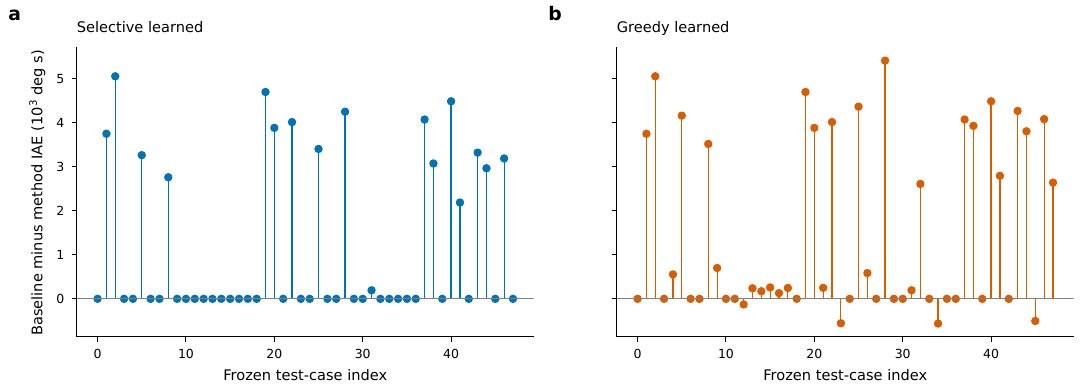}
\caption{Paired IAE changes for all 48 cases. Positive values favor the alternative. Selective control has 17 gains and 31 ties; greedy selection has a larger aggregate gain but four regressions.}\label{fig:paired}\end{figure}
\subsection{Context-dependent selection and avoided regressions}\label{sec:intelligence}
The outer learner's demonstrated role is to trade unconditional improvement against individual regressions. Fixed balanced and fixed fast reduce aggregate IAE by $3.61\%$ and $4.81\%$, but worsen 10 and 11 cases relative to the branch baseline. The full selector obtains $4.67\%$ with zero observed baseline regressions. Fixed fast has a $0.15\%$ lower mean IAE than the selector, so the value of intelligence here is not a new minimum of mean error. Greedy learned selection improves the mean further, by $5.84\%$, while regressing on four cases; the selective method gives up $1.23\%$ of mean IAE relative to greedy selection. All of these policies complete the tasks, so regression avoidance should not be called a demonstrated task-safety advantage.

The counterfactual comparison directly tests what rejection changes (Fig.~\ref{fig:intelligence}). In the 31 rejected contexts, forcing balanced or fast continuation causes 10 or 11 IAE regressions, respectively, with a worst forced-fast increase of $78.6\%$. These are actual complete rollouts, not predicted labels used for test-time tuning. For every test case, the full logged trajectory equals the corresponding executed fixed-policy counterfactual for its chosen action, including action 0. Thus the observed differences follow from policy choice under common conditions. Rejection also forgoes useful alternatives, consistent with its lower aggregate gain than fixed fast and greedy selection. The branch baseline already contains a learned ranking; this experiment assesses the additional outer selector, not an entirely learned architecture versus an entirely nonlearned one.

The useful selection region is limited (Table~\ref{tab:families} and Fig.~\ref{fig:families}). Alternatives are adopted in six principal-axis, five near-axis, four braking-boundary, and two noise/drift cases. Family IAE reductions are $7.75\%$, $6.75\%$, $4.42\%$, and $2.56\%$, respectively. Ordinary and general coupled cases all retain the baseline. The coupled result demonstrates retention, not improved coupled control. Predicted-versus-realized gains and all 48 selected actions are shown alongside the rejected-case counterfactuals, so favorable average performance is not substituted for a description of the operating domain.

\begin{figure}[tbp]\centering\includegraphics[width=\linewidth]{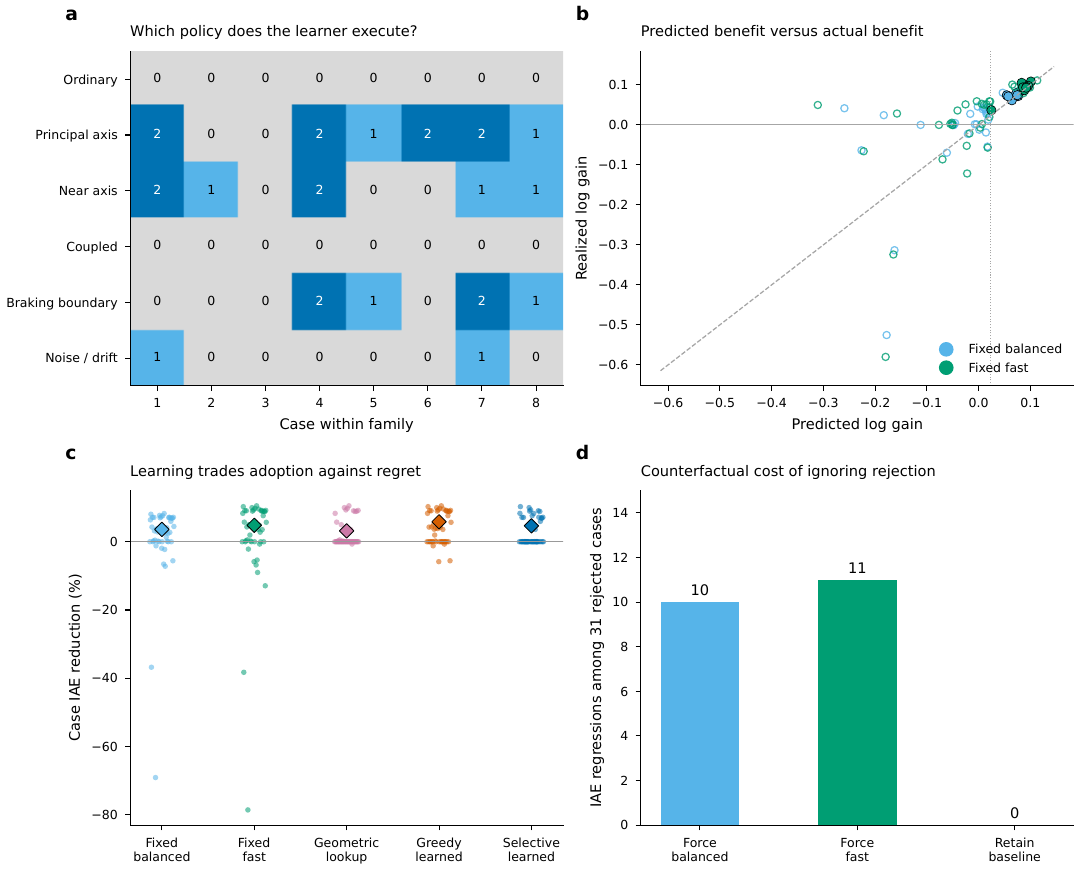}
\caption{Direct diagnostics of the intelligent outer selector. \textbf{a}, Actions in all 48 cases: 0 retains the baseline, 1 uses balanced speed, and 2 uses fast speed. Cases are ordered by original index within each family. \textbf{b}, Predicted and realized log IAE gains for both candidates in contexts with planning features; filled markers are adopted, hollow markers are not. The diagonal is perfect prediction; the vertical dotted line marks the gain margin plus subtraction and is only one acceptance condition. \textbf{c}, Individual case reductions (points) and aggregate reductions (diamonds); their definitions differ. \textbf{d}, Actual counterfactual IAE regressions when alternatives are forced in the 31 rejected cases. These diagnostics do not establish calibrated probabilities or safety guarantees.}\label{fig:intelligence}\end{figure}
\begin{table}[t]\centering\small
\caption{Family-level means and adoption. Each row has eight successes under both methods. Reductions are relative to the branch baseline.}\label{tab:families}
\begin{tabular}{lrrrr}\toprule
Family & Baseline IAE & Selective IAE & Reduction (\%) & Adoptions\\\midrule
Ordinary & 5,642.0 & 5,642.0 & 0.00 & 0/8 \\
Principal axis & 31,279.5 & 28,853.9 & 7.75 & 6/8 \\
Near axis & 32,777.6 & 30,564.3 & 6.75 & 5/8 \\
Coupled & 15,047.9 & 15,047.9 & 0.00 & 0/8 \\
Braking boundary & 45,037.7 & 43,048.9 & 4.42 & 4/8 \\
Noise and drift & 26,577.6 & 25,897.6 & 2.56 & 2/8 \\

\bottomrule\end{tabular}\end{table}
\begin{figure}[t]\centering\includegraphics[width=\linewidth]{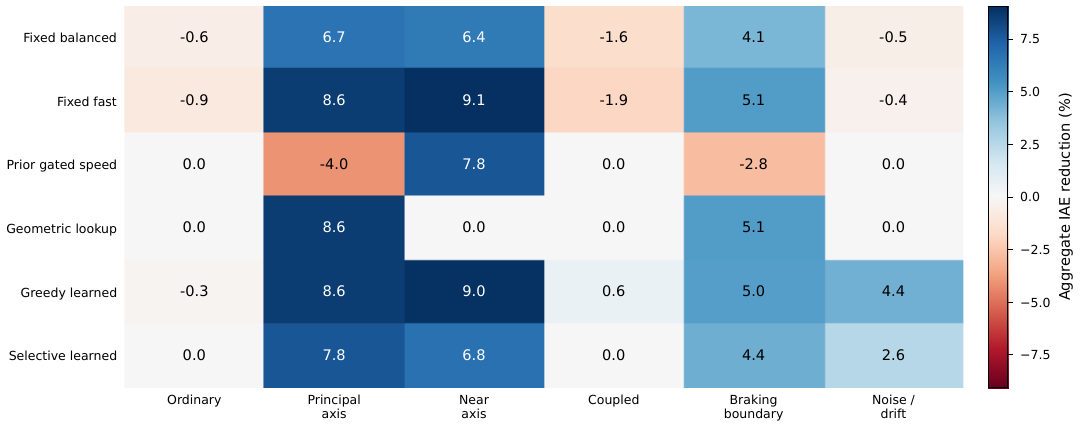}
\caption{Family-level aggregate IAE reductions relative to the branch baseline for each distinct policy comparator; each cell uses eight paired contexts. Positive values indicate improvement. Fixed speed alternatives improve some families while degrading others.}\label{fig:families}\end{figure}
\subsection{Control transients and mechanism ablations}\label{sec:supplement}
State and applied-torque histories show how the compared policies differ during capture (Fig.~\ref{fig:baselinecurves}). The near-axis case 038 is the median absolute gain among five adopted near-axis cases; noise/drift case 041 is the lower gain among two adoptions. Their three-axis MRP, angular-rate, and actual-torque traces appear in Figs.~\ref{fig:states} and \ref{fig:noise}. Shorter error accumulation accompanies changed braking and control transients, but not a pointwise reduction in every state or a universal effort reduction. Shadow-MRP jumps are coordinate switches, and torque magnitudes must not be interpreted as componentwise bounds.

The adverse case 032 is selected as the largest regression relative to prior gated speed. Here the selector retains the branch baseline, while the earlier gated controller captures with much lower accumulated error. Retention therefore prevents a departure from the incumbent but can retain its disadvantage relative to another controller. This deliberately unfavorable example complements the complete-suite statistics.

Component interventions distinguish implemented mechanisms from demonstrated benefits (Fig.~\ref{fig:mechanisms}). Commitment-only increases mean IAE by $0.79\%$ and worsens two cases. Removing commitment from the selected-speed controller leaves all 48 complete trajectories exactly unchanged. The present suite therefore supports selected speed adaptation without an incremental benefit from commitment. This frozen-controller intervention is not a factorial retraining study, so broader interactions remain uncharacterized.

Additional queries and score subtraction also lack an added aggregate benefit. The no-query variant has mean IAE $24\,830.0\;\mathrm{deg\,s}$ and a paired interval crossing zero relative to the full method; the no-subtraction variant has mean IAE $24\,812.1\;\mathrm{deg\,s}$ and no baseline regressions. The absence of regressions cannot therefore be uniquely attributed to $q$. Separate failure/gain modeling is an interpretable design choice, but no matched single-objective ablation establishes its independent performance benefit. Figure~\ref{fig:ablationcurves} retains all four module interventions; their cases are chosen by the largest absolute IAE difference from the full method, explicitly as diagnostic extremes rather than representative averages.

\begin{figure}[tbp]\centering\includegraphics[width=\linewidth]{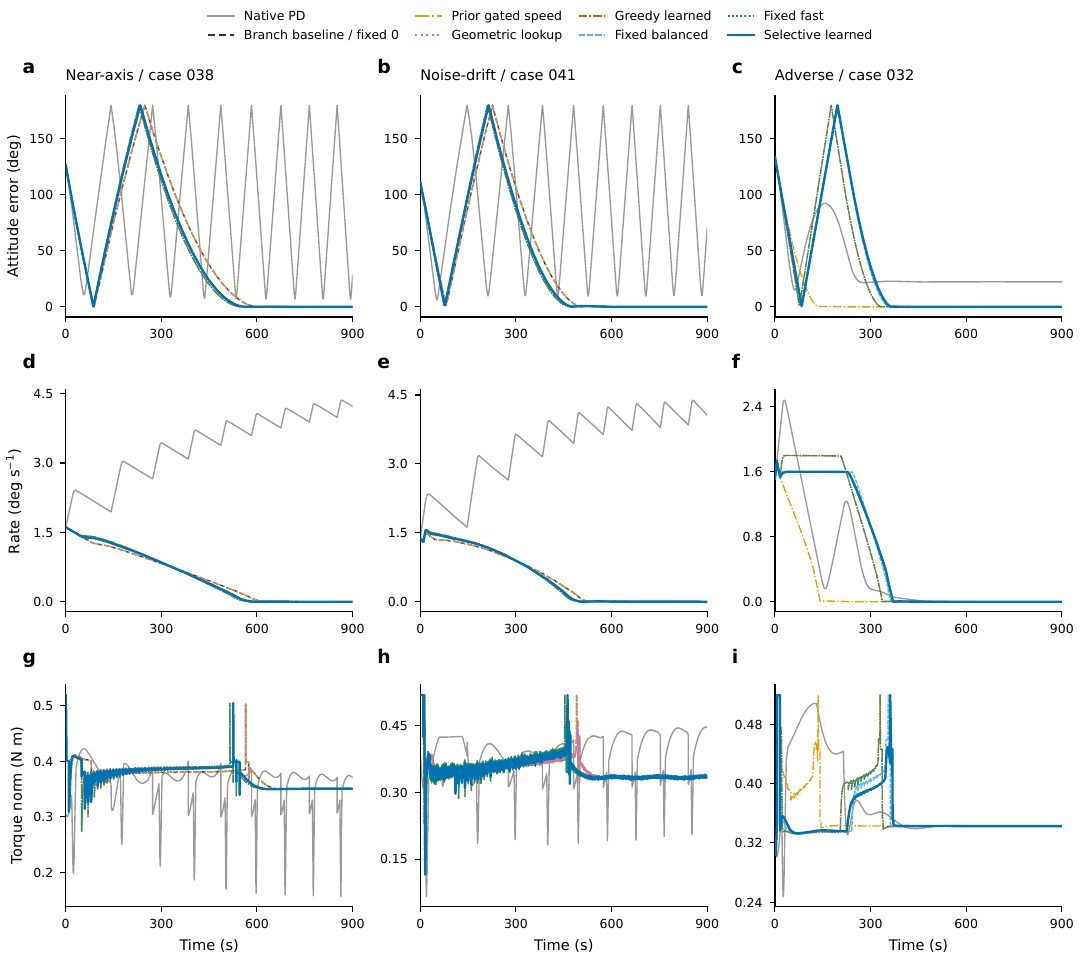}
\caption{Control transients for all distinct policy comparators: native PD reference, branch baseline, prior gated speed, geometric lookup, greedy learned selection, fixed balanced, fixed fast, and full selective learning. Columns show near-axis case 038 (median adopted near-axis gain), noise/drift case 041 (lower adopted gain), and adverse case 032 (largest regression versus prior gated speed). Rows show attitude error, angular-rate magnitude, and actual torque norm. Overlapping trajectories are not omitted. All curves use unsmoothed $0.1\,$s logs; componentwise torque limits should not be inferred from torque norms. Data/correction ablations appear separately in Fig.~\ref{fig:ablationcurves}.}\label{fig:baselinecurves}\end{figure}
\begin{figure}[tbp]\centering\includegraphics[width=\linewidth]{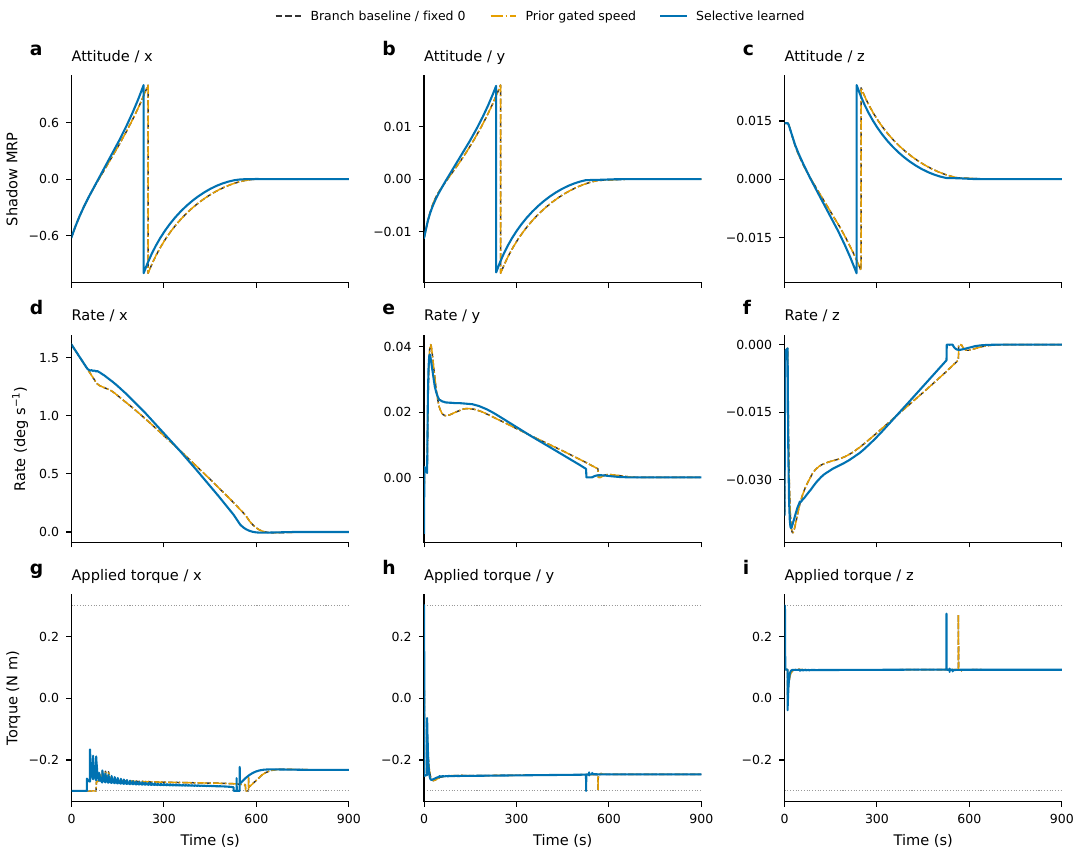}
\caption{Case 038 component histories: shadow MRPs (top), angular velocities (middle), and actual applied torques (bottom). Dotted gray lines mark $\pm0.3\;\mathrm{N\,m}$ component limits. MRP jumps are shadow-coordinate switches, not physical attitude discontinuities. Overlapping curves remain visible as a single trace. All curves use unsmoothed $0.1\,$s logs.}\label{fig:states}\end{figure}
\begin{figure}[tbp]\centering\includegraphics[width=\linewidth]{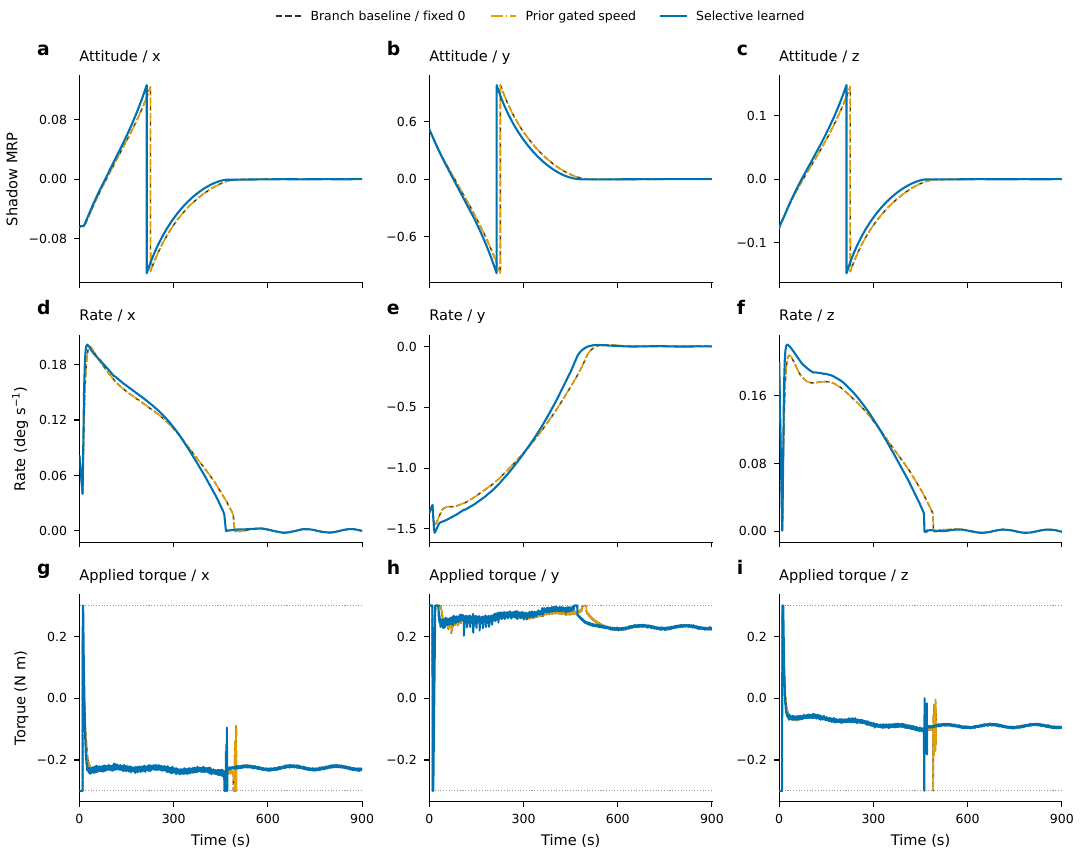}
\caption{Noise/drift case 041, selected as the lower of the two adopted gains in this family. Rows show shadow MRPs, angular velocities, and actual applied torques. The same colors, line styles, time span, and component torque limits as Fig.~\ref{fig:states} permit direct reading across the two examples.}\label{fig:noise}\end{figure}
\begin{figure}[tbp]\centering\includegraphics[width=\linewidth]{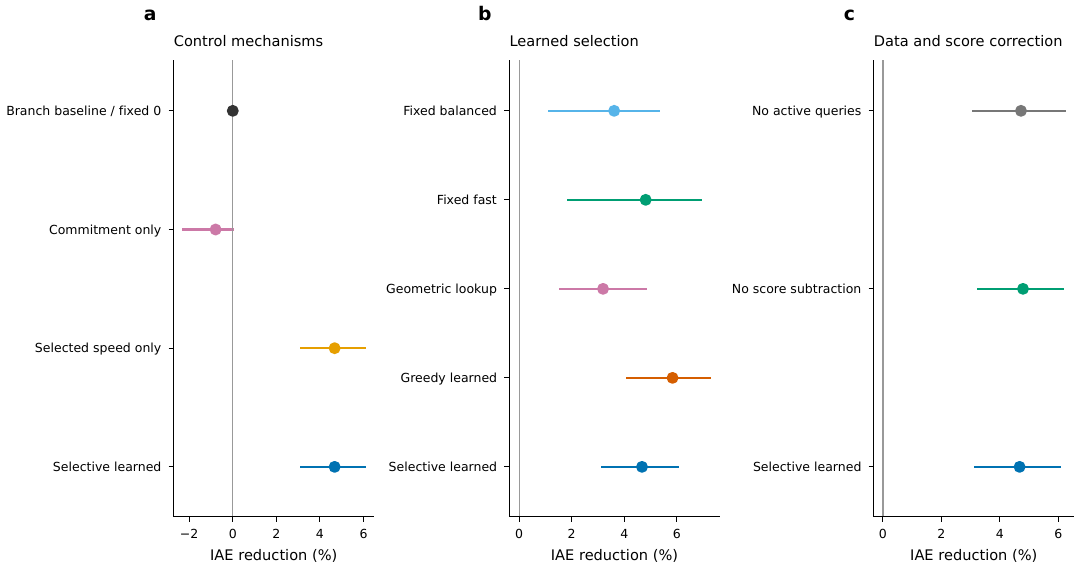}
\caption{Evidence for the proposed mechanisms. \textbf{a}, Interventions on branch commitment and selected speed; commitment alone does not improve performance, and removing commitment from the selected-speed controller leaves trajectories unchanged. \textbf{b}, Learned selection versus fixed alternatives, geometric lookup, and greedy learning. \textbf{c}, Data acquisition and score-correction ablations show no extra aggregate gain for the full pipeline. Points show aggregate reductions relative to the branch baseline; intervals use the same paired-case bootstrap as Fig.~\ref{fig:baselineoverview}. Intervals describe context variation, not repeated training.}\label{fig:mechanisms}\end{figure}
\begin{figure}[tbp]\centering\includegraphics[width=\linewidth]{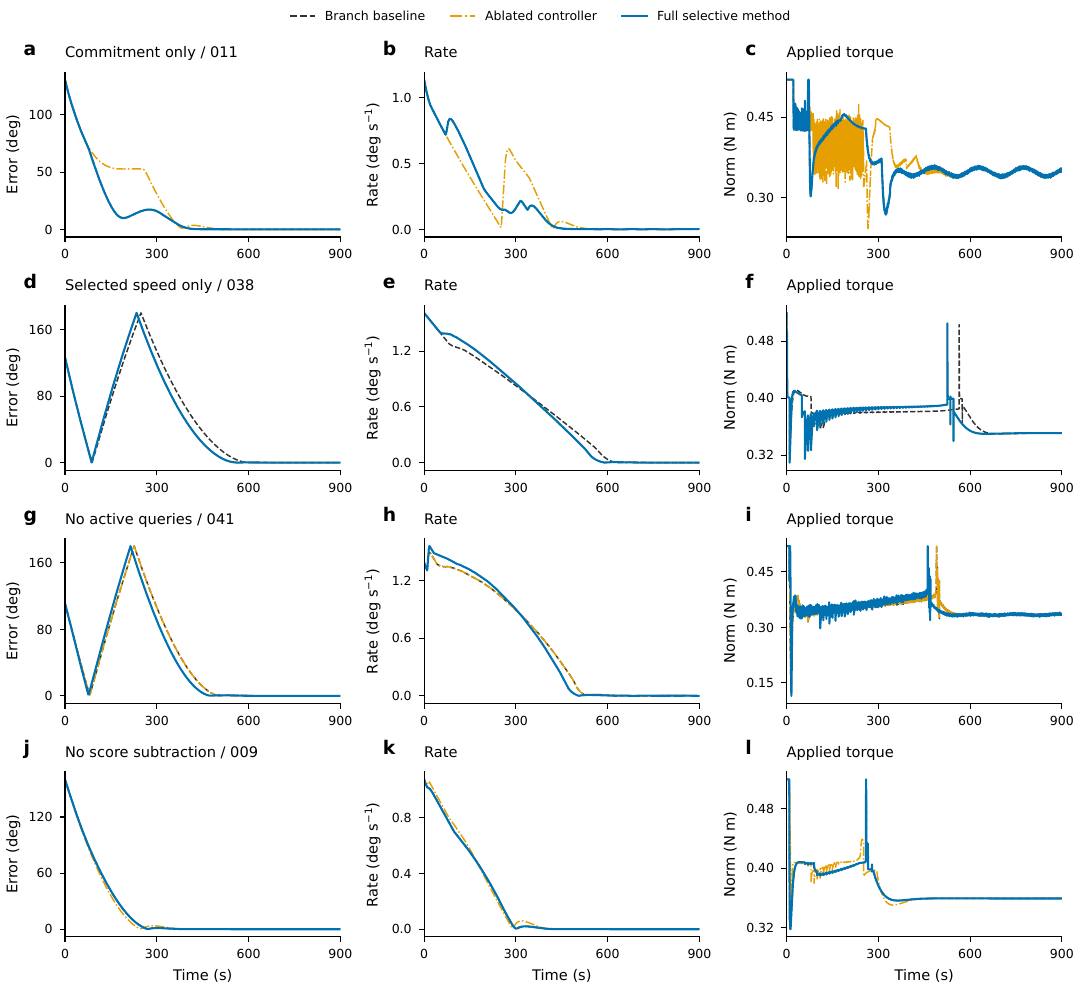}
\caption{Module ablation transients, compared with the branch baseline and full selector. Rows show commitment-only, selected-speed-only, no-active-query, and no-score-subtraction controllers. For each row, the case with the largest absolute IAE difference from the full method is selected, explicitly as a diagnostic extreme; the identical selected-speed-only variant uses case 038. Columns show attitude error, rate magnitude, and actual applied torque norm. Original $0.1\,$s logs are unsmoothed. Full three-axis state and torque histories for two examples appear in Figs.~\ref{fig:states} and \ref{fig:noise}.}\label{fig:ablationcurves}\end{figure}
\subsection{Computational cost and scope of the evidence}\label{sec:discussion}
The outer selection is performed once per episode, but the full control architecture still incurs inherited branch-planning cost. The main configuration uses 215 finite branch forecasts versus 233 for the baseline. Warm selector inference averages $33.6\,$ms in a verification run under concurrent CPU load; the maximum main control callback is $0.440\,$s. The simulation pauses while computing and injects no execution delay. These observations do not establish a $0.1\,$s onboard implementation or a controlled hardware speedup.

The evaluation includes physically difficult braking-boundary cases and modest noise/drift, but its robustness evidence has defined limits. The two $z$-axis boundary test cases exceed the estimator's disturbance prior, while Gaussian measurement noise and time variation also depart from deterministic constant-disturbance assumptions. Successful finite simulations under these mismatches do not prove robust invariance. Empty-set resets remain recovery heuristics. Only one candidate failure occurs in development, and all test portfolio variants succeed; failure-score calibration and task-failure discrimination are consequently not established. A threshold of 0.02 cannot be read as a 2\% failure probability.

The benchmark uses diagonal constant inertia, one noise realization per context, and one frozen training procedure. It omits non-diagonal inertia, actuator dynamics, computation delays, flexible structures, pointing exclusion cones, and orbital coupling. Published MPC and explicit-governor implementations have not been compared under matched budgets. The result is therefore restricted to this portfolio and suite. Additional training seeds, adversarial failure discovery, and matched published-controller tests are the next evidence needed before stronger generalization or safety claims.

\clearpage
\section{Conclusions}\label{sec:conclusion}
Paired-outcome learning provides a concrete way to modify a torque-limited spacecraft capture controller selectively. Physically interpretable continuations supply complete control alternatives; matched rollouts label task failure separately from successful-pair tracking gain; and rejection executes the unchanged branch baseline. On 48 test contexts, the selector reduces aggregate IAE by $4.67\%$ with 17 improvements and 31 exact ties. Fixed fast achieves a slightly lower mean error but produces 11 baseline regressions, demonstrating the observed trade-off between unconditional speed tuning and selective execution.

The improvement is concentrated in principal-axis, near-axis, braking-boundary, and selected noise/drift maneuvers. General coupled cases retain the baseline. Branch commitment, active querying, and score subtraction show no independent aggregate advantage in the reported interventions, and no safety guarantee follows for accepted policies. Future validation should prioritize coupled dynamics, repeated training, actuator and timing effects, and matched predictive-control or governor baselines. These tests would determine whether the present empirical improvement extends beyond its evaluated operating domain.

\section*{Reproducibility and preparation disclosure}
The accompanying experiment archive retains policy implementations, seeds, models, labels, freeze records, and full-resolution metrics. The manuscript evidence bundle snapshots the source JSON files and their checksums. Earlier distributed trajectory samples use $2\,$s spacing; the revised figure evidence includes $0.5\,$s samples for the displayed cases; reported metrics use original $0.1\,$s logs, which the experiment scripts regenerate. No public repository or archival DOI is asserted.

Generative AI tools assisted algorithm implementation, experiment orchestration, analysis, and manuscript preparation. Results are linked to executed simulations. Tool assistance does not replace scientific review or responsibility for the final submission.

\appendix
\section{Estimator and inherited planner details}\label{app:estimator}
For cyclic indices $(i,j,k)$, integrating~\eqref{eq:euler} gives
\begin{equation}
J_i\Delta\omega_i-d_it=U_i+(J_j-J_k)W_i,\quad U_i=\int_0^t u_i\,ds,\quad W_i=\int_0^t\omega_j\omega_k\,ds.
\end{equation}
Accumulated input and trapezoidal rate-product integrals are used. Every $10\,$s, the prior range of $J_j-J_k$ bounds the right-hand side, enlarged by $\epsilon_i(t)=2(5\times10^{-4})\overline J_i+0.005t$ in angular-momentum units. Two half-plane intersections update the correlated $(J_i,d_i)$ polygon; fewer than three vertices resets the prior. This allowance does not bound unbounded Gaussian noise or arbitrary disturbance drift.

For projected direction $\bm n$, the braking routine uses $\bm u_b=-\bar{\bm u}\odot\sgn(\bm n)$ and
\begin{align}
A_i&=\max_{(J_i,d_i)\in\mathcal P_i}\frac{n_i(u_{b,i}+d_i)}{J_i},\\
B_i&=\max_{J\in\mathcal V}\frac{n_i(J_j-J_k)\omega_j\omega_k}{J_i},\\
b(\bm n,\bm\omega)&=\max\{0,-\textstyle\sum_i(A_i+B_i)\},
\end{align}
where $\mathcal V$ comprises the eight prior inertia-box vertices. Separate maxima overbound the projected expression under its fixed assumptions; they do not establish a complete multi-axis capture result.

The inherited predictor evaluates an observer-nominal scenario and a weak-initial-braking polygon-vertex scenario. Its score adds predicted elapsed time, 0.002 times predicted IAE, and a 10,000 penalty for failed finite-rollout checks. Its internal terminal dwell is $10\,$s, whereas measured task completion uses $60\,$s. This planner score is distinct from the measured IAE target used for portfolio learning. The complete original branch model and data remain in the experiment archive.

\section{Feature ordering and fitting details}\label{app:features}
Let $e_0=\theta-2\pi\operatorname{round}(\theta/2\pi)$ and $v_s=2^\circ/\mathrm{s}$. The first 32 features are:
\begin{enumerate}
\item $e_0/\pi$, $v/v_s$, and $\bm\omega/v_s$ (5 entries);
\item $\bm n$, $\widehat{\bm d}\oslash\bar{\bm u}$, and $\bar{\bm u}/0.3$ (9 entries);
\item forward/reverse braking divided by $0.001\;\mathrm{rad/s^2}$, transverse rate divided by $v_s$, and $(\bm\omega\times J_n\bm\omega)\oslash\bar{\bm u}$ (6 entries);
\item per-axis polygon extrema $J_i^{\min}/J_{n,i}-1$, $J_i^{\max}/J_{n,i}-1$, $d_i^{\min}/\bar u_i$, $d_i^{\max}/\bar u_i$ (12 entries).
\end{enumerate}
The last two base entries are $(\theta-\theta_\star)/(2\pi)$ and braking status. Nine added entries are: stopping-distance ratio clipped at 20 then divided by 5; maximum $|n_i|$; transverse/projected rate ratio with normalized denominator floor 0.02; normalized gyroscopic norm; largest normalized disturbance estimate; its projection on $\bm n$; absolute target distance; braking status; and normalized total rate.

Context bootstrapping precedes fourfold augmentation, so copies are never treated as independent cases for splitting. Gain regressors fit successful pairs only. Hyperparameters are selected separately for the main and no-query models using common validation and calibration contexts.

Geometric lookup bins are: $\max|n_i|\geq1-10^{-8}$ and normalized transverse/projected ratio below $10^{-5}$; otherwise dominance above 0.95 and ratio below 0.2; and the remainder. Validation minimizes failure count then summed IAE within each bin. This baseline uses the same paired label table rather than inaccessible test-family labels.

\section{Reproduction and statistics}
The experiment root contains \texttt{environment.py}. Final policies are stored in
\begin{quote}\small\path{research_baselines/portfolio/}\end{quote} Run \texttt{run.py}, \texttt{audit.py}, and \texttt{report.py} there to regenerate test logs and verify metrics. \texttt{data.py} and \texttt{active.py} generate labels; \texttt{portfolio\_model.py} fits the model, with \texttt{--base-only} for its no-query counterpart. Retraining requires a new freeze and fresh tests before claiming independent generalization.

The model SHA-256 is
\begin{center}\small\ttfamily
49cf6ad818705ca14f6e8db21f6e711718\allowbreak45862bb49450e591288dbb2fe0b335.
\end{center}
The evidence manifest identifies the exact JSON snapshots used for tables. Paired bootstrap seed 834092 resamples case indices jointly across controllers, takes the mean absolute IAE difference, and reports 2.5th/97.5th percentiles. Aggregate reductions use $(\sum C_0-\sum C_a)/\sum C_0$, not the mean of per-case percentages.

Baseline and best-fixed policies are duplicates in this experiment. All eight original configurations and four new portfolio diagnostics have 48 successes; the native PD reference has zero task completions and uses different feedback gains without disturbance estimation. Failure probabilities among the portfolio policies cannot be discriminated. Neither observed zero regressions nor Proposition~\ref{prop:identity} guarantees performance for accepted actions. Complete rejection of the general coupled family and the negative ablations are substantive parts of the result.

\section{Training-only query rule}\label{app:query}
A further 96-context pool is observed through $30\,$s baseline prefixes. A preliminary model fitted only on the initial training split scores contexts with available decision features by
\begin{equation}
Q(x)=\max_{a=1,2}s_a(x)+\frac{0.04}{3}\sum_{a=0}^2 r_a(x)(1-r_a(x))+0.02\min\{D(x),3\}.\label{eq:query}
\end{equation}
Here $r_a$ is a classifier score, $s_a$ is ensemble dispersion, and $D$ is a normalized distance defined below. Two queries are taken per family: one maximizing $Q$, and a second maximizing $Q$ plus 0.03 times its symmetry-aware distance from the first. The 12 queried contexts receive full three-policy labels and join training only. This is a query heuristic, not an optimal experimental-design claim.

The acquisition weights are heuristic and were not retuned on the test set. The no-query ablation is necessary because additional labeling does not itself establish a performance contribution.

\bibliographystyle{plainnat}
{\raggedright\bibliography{references}}
\end{document}